\documentclass[journal,11pt,draftclsnofoot,onecolumn]{IEEEtran}
\usepackage{hyperref}
\usepackage{amsmath}
\usepackage{amsfonts}
\usepackage{amssymb}
\usepackage{amscd}
\usepackage[dvips]{graphicx}
\usepackage{tikz}
\def\BibTeX{{\rm B\kern-.05em{\sc i\kern-.025em b}\kern-.08em
    T\kern-.1667em\lower.7ex\hbox{E}\kern-.125emX}}

\usepackage{color}
\usepackage{amsthm}
\usepackage{newlfont}
\usepackage{color}
\usepackage{setspace}
\usepackage{subfigure}
\newtheorem{thm}{Theorem}

\newtheorem{lem}{Lemma}

\newtheorem{defn}{Definition}

\usepackage{cite}
\newcommand{\beq}{\begin{equation}}
\newcommand{\eeq}{\end{equation}}
\newcommand{\beqa}{\begin{eqnarray}}
\newcommand{\eeqa}{\end{eqnarray}}

\newcommand{\norm}[1]{\Vert#1\Vert}

\begin{document}
 
%\author{\authorblockN{{Prashant Khanduri\IEEEauthorrefmark{1}, Bhavya Kailkhura\IEEEauthorrefmark{1}, \IEEEauthorrefmark{2}, and Pramod K. Varshney\IEEEauthorrefmark{1}}}\\
 
\title{\textbf{Universal Collaboration Strategies for Signal Detection: A Sparse Learning Approach}}
\author{\IEEEauthorblockN{Prashant Khanduri\IEEEauthorrefmark{1}, Bhavya Kailkhura\IEEEauthorrefmark{1}, Jayaraman J. Thiagarajan\IEEEauthorrefmark{2} and Pramod K. Varshney\IEEEauthorrefmark{1}}
\IEEEauthorblockA{\IEEEauthorrefmark{1}Department of EECS, Syracuse University, NY, 13244 USA  \\
Email:\{pkhandur,bkailkhu,varshney\}@syr.edu}
\IEEEauthorblockA{\IEEEauthorrefmark{2}Lawrence Livermore National Laboratory\\
Email:jayaramanthi1@llnl.gov}}

\maketitle
%##########################################################################
%##########################################################################
\begin{abstract}
This paper considers the problem of high dimensional signal detection in a large distributed network whose nodes can collaborate with their one-hop neighboring nodes (spatial collaboration). We assume that only a small subset of nodes communicate with the Fusion Center (FC). We design optimal collaboration strategies which are universal for a class of deterministic signals. By  establishing the equivalence between the collaboration strategy design problem and sparse PCA, we solve the problem efficiently and evaluate the impact of collaboration on detection performance.
% we seek the answers to the following questions:
%$1)$ How much do we gain from optimizing the collaboration strategy? $2)$ What is the effect of dimensionality reduction for different sparsity constraints? $3)$ How much do we lose in terms of detection performance by adopting a universal system?
\end{abstract}
%##############################################################################
%#############################################################################
\section{Introduction}
In a conventional signal detection problem, the goal is to design a system for detecting a specific signal of interest~\cite{Kay_1998}. The performance of such systems degrades if the signal evolves over time or for other known signals. Due to the advent of Big Data applications, modern detection systems are expected to perform signal detection tasks for different signal models. Hence, it is desirable to build a universal system which is flexible enough to generalize to several signal models. This paper considers a Wireless Sensor Network (WSN) consisting of a number of sensors and a FC. WSNs often operate with severe resource limitations. Consequently, minimizing the system complexity in terms of communication is critical \cite{Akyildiz_ComMag_2002}. For example, resources can be conserved if the nodes do not transmit irrelevant or redundant data. 
%However, it is usually not known in
%advance which measurement elements of the measurement vector are useful for the detection task. 
Such transmissions can be avoided through dimensionality reduction \cite{Fodor02DRSurvey}. The problem of dimensionality reduction at local sensors was considered in the context of distributed estimation in \cite{Schizas_TSP_2007,Xiao_Cui_TSP_2008} and distributed detection in \cite{Fang_Li_TSP_2012,Fang_Liu_SPL_2013}

Moreover, in certain systems, sensors can collaborate with their one-hop neighbors and form a network wide low dimensional projection of the observed signal. The resulting low-dimensional projection of measurements is transmitted by a small subset of sensors to the FC. Some variants of this idea have been used in the distributed estimation literature~\cite{Fang_TWC_2009,Kar_TIT_2013,Liu_TSP_2015}.

In large networks, it is not always feasible to modify the collaboration strategy for each and every sensor for different signal detection tasks. Moreover, the sensors are designed to acquire data pertinent to a hypothesis test without being aware of the signal model. In such scenarios, a practical approach is to design a universal collaboration strategy which is effective for a broad class of signals. To the best of our knowledge, there is no work which considers the design of cost constrained linear collaboration among sensor nodes for detection problems even for a single signal of interest. In this letter, we take some first steps towards the design of universal collaboration strategies for high-dimensional signal detection and seek to answer the following questions: $1)$ How much do we gain from optimizing the collaboration strategy? $2)$ What is the effect of dimensionality reduction for different sparsity constraints? $3)$ How much do we lose in terms of detection performance by adopting a universal system?

%We show that the problem of designing an effective collaboration strategy can be viewed as dimensionality reduction, wherein the goal is to reduce signal dimensions by collaboration such that detection performance is maximized. In particular, we establish an equivalence to Principal Component Analysis (PCA) \cite{Smith_2002}, a popular linear dimensionality reduction technique. Though collaboration is an effective strategy, it results in an increased power budget. Consequently, we impose sparsity constraints to control the cost of collaboration.

%We demonstrate that the proposed approaches can capture the information relevant for many signal detection applications. In addition, we show that the required number of measurements scale efficiently with the complexity of both the signal class.
%
%The main contributions of the paper can be summarized as follows:
%
%\begin{itemize} 
%\item We propose a universal signal detection framework with spatial collaboration and define the cumulative deflection coefficient (C-DC) metric to characterize its detection performance.
%\item We establish the equivalence between C-DC maximization and Principal Component Analysis (PCA).
%\item We empirically characterize the trade-off between the achievable performance of the proposed framework and the cost of collaboration and dimensionality reduction.
%\item Finally, by defining a metric to quantify the cost of universality, we study the price one pays for universality with respect to the inference performance.
%\end{itemize}
In this letter, we show that the problem of designing an effective collaboration strategy can be viewed as dimensionality reduction, wherein the goal is to reduce signal dimensions by collaboration such that performance is maximized. In particular, we establish an equivalence to Principal Component Analysis (PCA) \cite{Smith_2002}, a popular linear dimensionality reduction technique. Though collaboration is an effective strategy, it directly results in an increased power budget, and a complex network design. Consequently, we propose to impose sparsity constraints to control the cost of collaboration.

The main contributions of the paper can be summarized as follows:

\begin{itemize} 
\item We propose a universal signal detection framework with spatial collaboration and define the cumulative deflection coefficient (C-DC) metric to characterize its detection performance.
\item We establish the equivalence between C-DC maximization and Principal Component Analysis (PCA).
\item We empirically characterize the trade-off between the achievable performance of the proposed framework and the cost of collaboration and dimensionality reduction.
\item Finally, by defining a metric to quantify the cost of universality, we study the price one pays for universality with respect to the inference performance.
\end{itemize}
%################################################################################################
%################################################################################################
\section{Collaboration Strategies for Signal Detection}

\subsection{Hypothesis Testing}
Consider a distributed sensor network designed to determine the presence or the absence of a high-dimensional signal $\mathbf{s}$. $N$ sensors each sensing a scalar variable combine to sense an $N$ dimensional signal $\mathbf{s}$,
\begin{align} \label{eq:VSD}
H_0: ~~\mathbf{x}&=\mathbf{n},  \nonumber \\
H_1: ~~\mathbf{x}&= \mathbf{s} + \mathbf{n}, 
\end{align}
where, $\mathbf{x} \in \mathbb{R}^{N}$ is the observed signal,  $\mathbf{n} \sim \mathcal{N}(0, \sigma^2 \mathbf{I}_N)$ is the additive white Gaussian noise (AWGN) with covariance $\sigma^2\mathbf{I}_N $ and $\mathbf{s} \in \mathbf{R}^N$ is the signal of interest. 

\subsection{Collaboration for Distributed Detection}
\subsubsection{Distributed Detection}
Consider a parallel network with $N$ sensing nodes where each node can forward its observation of the signal of interest $\mathbf{s}$ in noise to the Fusion Center through a noiseless communication link. 
%For simplicity, we assume that each node is capable of observing one element/feature of the signal of interest $\mathbf{s}$. 
The FC then processes the observed data and decides in favor of $H_0$ or $H_1$. However, in large networks, due to a variety of reasons including power budget and network design,
 %\note[jay]{i think network design will include the node locations}
it may not always be possible for all the sensing nodes to communicate to the FC. We propose to alleviate this fundamental challenge by using collaboration schemes.
  
\subsubsection{Collaboration Schemes}
We begin by assuming that only a subset $M$ of the $N$ sensing nodes, where $M<<N$, are allowed to transmit to the FC to possibly conserve energy.
In addition, these nodes have the ability to update their observations through collaboration, which refers to the process of combining their observations with those from their one-hop neighboring nodes. 
Without loss of generality, we assume that the nodes are ordered such that only the first $M$ nodes can communicate with the
FC. We define $\mathbf{W}\in \mathbb{R}^{M \times N}$ as the collaboration matrix whose elements correspond to the weights to combine the node observations. 
Note that, $\mathbf{W}$ projects the high-dimensional signal $\mathbf{x}\in \mathbb{R}^N$ onto $\mathbf{y} \in \mathbb{R}^M$ as $\mathbf{y}=\mathbf{W} \mathbf{x}$, where $M \leq N$, as shown in Fig. \ref{fig:system_model}.

 The FC performs a hypothesis test and infers a global decision about the signal of interest solely based on the $M$ low-dimensional measurements $\mathbf{y}$. 
The goal of the designer is to design an optimal collaboration matrix $\mathbf{W}$ such that the detection performance of the system is maximized.
 \begin{figure}[t]
 \begin{center} 
 {\includegraphics[scale=0.55]{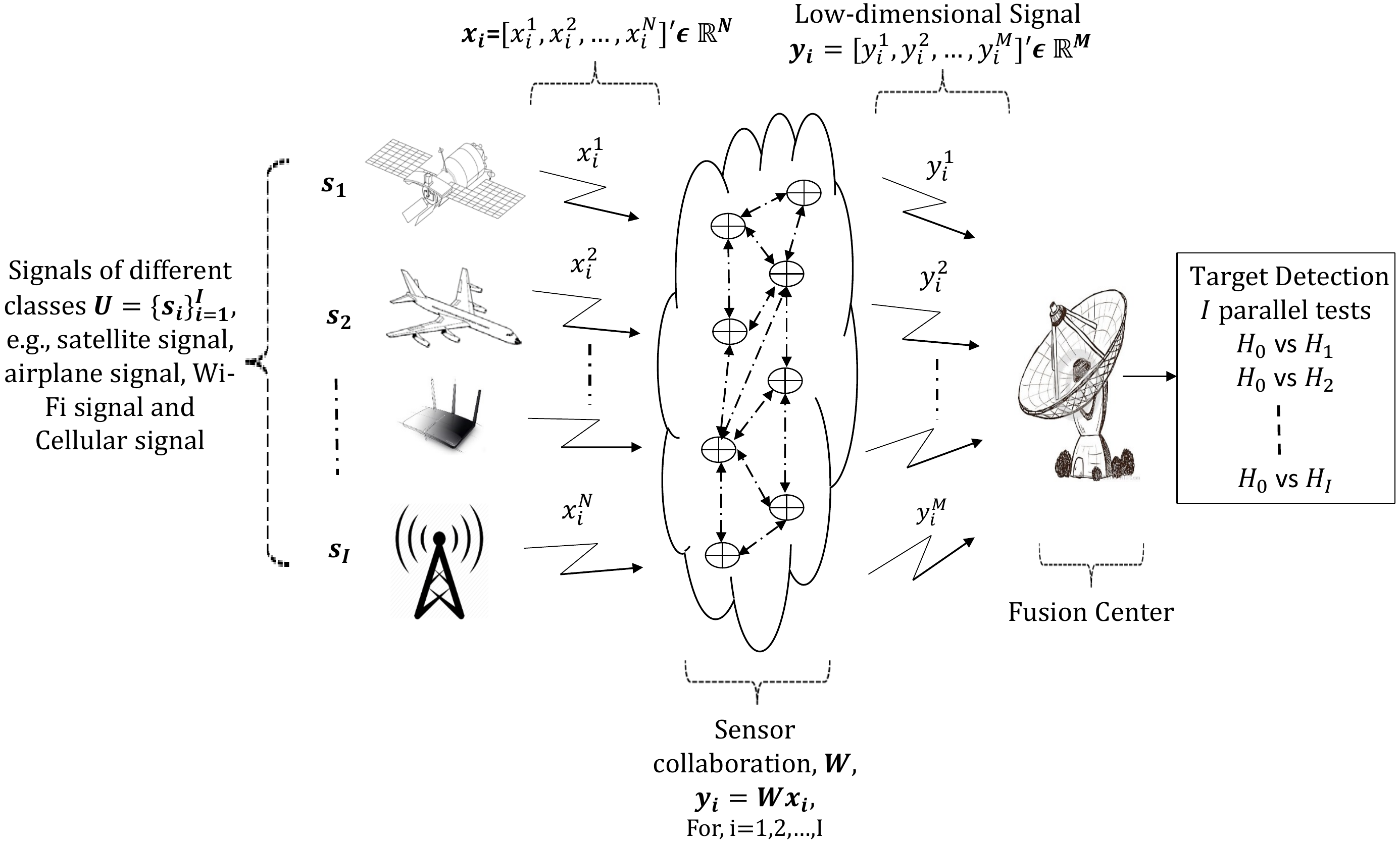}} 
 \caption{System model for the proposed distribution detection framework. It is assumed that only $M$ out of the $N$ total nodes can transmit to the Fusion Center, and they have the ability to collaborate with their one-hop neighboring nodes. The spatial collaboration process is modeled as a linear projection with the weight matrix $\mathbf{W}$.}
 \label{fig:system_model}
 \end{center} 
 \vspace{-0.1in}
 \end{figure} 
 For clarity of exposition, we first formulate this problem for the case of a single signal of interest. In this formulation, we use deflection coefficient as the performance metric. It is well known that the maximization of the deflection coefficient at the FC is equivalent to the minimization of the probability of error. The design problem for detecting a known signal $\mathbf{s}$ is  
\begin{align}
\underset{W}{\text{maximize}} ~~  \mathbf{s}^T \mathbf{W}^T \left(\mathbf{W}\mathbf{W}^T \right)^{-1} \mathbf{W} \mathbf{s}.
\end{align}
%Note that, we are interested in a universal system design which is flexible enough to accommodate the detection of different signals without changing the collaboration strategy.
As we will see later, the solution to this problem is trivial and can be handled as a special case of Lemma \ref{lem:PCA} for $I=1$.
In the next section, we generalize this setup to obtain universal collaboration strategies for a broader class $\mathcal{U}$ of signals $\{\mathbf{s}_i\}_{i=1}^{I}$. We assume that the signal $\mathbf{s}$ belongs to a class $\mathcal{U}= \{\mathbf{s}_i \}_{i=1}^I$ of signals where $\mathbf{s}_i$ are deterministic and the FC has the knowledge of the elements of set $\mathcal{U}$ (Fig. \ref{fig:system_model}). This model has practical applications in the context of several big data problems and has also been considered in \cite{Davenport_Baraniuk_2010, Kailkhura_Arxiv_2015}.

For clarity of exposition, we illustrate one instance of the application in Figure \ref{fig:system_model}. In Figure \ref{fig:system_model}, the goal of the FC is to detect the presence or the absence of signals emitted by $I$ parallel data sources where signals come from a class $ \mathcal{U}=\{\mathbf{s}_i\}_{i=1}^I$. The signal model $\mathcal{U}$ is known to the FC, i.e., in the figure the FC knows that it is detecting signals/objects such as satellite, airplane, cellular base station signal and Wi-Fi signal all of which are assumed to be deterministic for this example specifically.

%This $\mathbf{W}$ is optimal only for the fixed signal $\mathbf{s}$. 
%However, we are often
%nterested in a universal system design which is flexible enough to accommodate several signal detection tasks.

\subsection{Universal Collaboration Strategies}
%Before we describe the proposed universal collaboration schemes, we define some metrics to quantify the performance in this case.

\subsubsection{Performance Metrics}
We assume that the signal under the alternate hypothesis $H_1$ can come from a set of equally probable signals, $\{\mathbf{s}_i\},\;i=1,\cdots,I$. To characterize the detection performance of the system, we define the following metric:
 
\begin{defn}{\textbf{(Cumulative Deflection Coefficient)}}\label{def:CDC}
We define Cumulative Deflection Coefficient (C-DC) for a 
signal class $\mathcal{{U}}=\{\mathbf{s}_i\}_{i=1}^I$ as
 \begin{align}
  \text{C-DC}= \sum_{i=1}^I \mathbf{s}_i^T \mathbf{W}^T \left(\mathbf{W}\mathbf{W}^T \right)^{-1}\mathbf{W}\mathbf{s}_i,
  \end{align}which is the summation of individual deflection coefficients for each $ \mathbf{s}_i$, 
  \label{def:TotalDC}
  \end{defn}We propose to maximize C-DC, which takes into account the cumulative detection performance of the system for all $I$ signals. Note that a universal collaboration design will incur a certain level of loss in terms of detection performance. For this purpose, we define a metric to measure the cost of universality that quantifies the performance loss of the system as $I$ increases.
 
 \begin{defn}{\textbf{(Cost of Universality)}}\label{def:CU}
The Cost of Universality ($C_u(I)$) is the performance loss when using a single collaboration strategy for a set of $I$ signals. It is characterized by
\begin{align}
C_u(I)= \frac{\text{C-DC}}{\sum_{i=1}^I \mathbf{s}_i^T \mathbf{s}_i}
\end{align} 
 as the number of signals $I$ increases. %, where C-DC is the Cumulative Deflection Coefficient from Definition~\ref{def:TotalDC} for optimal $\mathbf{W}$. 
 \end{defn}The denominator represents the summation of deflection coefficients when the collaboration strategy is optimized separately for each signal (Lemma 2 of \cite{Kailkhura_2015}). On the other hand, the numerator C-DC is the deflection coefficient when we use a universal collaboration strategy $\mathbf{W}$ for all $I$ signals.  
 %\note[jay]{I find this term "cost of universality" very non-intutive. The way it is defined, higher the value of $C_u(p)$ it is desirable. but to say high cost is desirable sounds contradictory. }.
  Now, using the Cauchy-Schwartz inequality, we get
 \begin{align}
 \norm{\mathbf{P_w}\mathbf{s}_i}_2^2 &\leq \norm{\mathbf{P_w}\mathbf{s}_i}_2 \norm{\mathbf{s}_i}_2,
 \end{align}
 where $\mathbf{P_w}=\mathbf{W}^T ( \mathbf{W}\mathbf{W}^T)^{-1} \mathbf{W}$. Hence, $\norm{\mathbf{P_w}\mathbf{s}_i}_2 \leq \norm{\mathbf{s}_i}_2$, which implies that $\text{C-DC} \leq \sum_{i=1}^I \norm{\mathbf{s}_i}_2^2=\sum_{i=1}^I \mathbf{s}_i^T\mathbf{s}_i$.
 
When the $i$th sensor shares its information as indicated by the collaboration matrix $\mathbf{W}$, it will incur a finite cost $\gamma_i$ arising due to practical considerations such as power consumption. In practice, it is desirable to minimize this cost, referred to as the cost of collaboration.

\begin{defn}{\textbf{(Cost of Collaboration)}}\label{def:CC}
 We define the cost of collaboration in our detection system as $C_c=\sum_{i=1}^M |\gamma_i|$, where $\gamma_i$ is the cost for communication as specified by the $i^{\text{th}}$ row of the collaboration matrix $\mathbf{W}$. 
 \end{defn}
 Broadly speaking, there is a trade-off between the detection performance and the cost efficiency of a system. As the number of nodes capable of transmitting to the FC ($M$) increases, the detection performance will improve. On the other hand, if the collaboration cost $\gamma_i$ increases, the detection performance is expected to degrade, as less number of resources (communication links) can be used under a fixed cost budget. 
 
 %Therefore, there exists a trade off in selecting $M$ and $\gamma$ for a given system, which will be evident in later sections.
%#############################################################################################
%############################################################################################# 
  \section{Optimal Universal Collaboration Strategies for Signal Detection}
%First, we propose a simple collaboration scheme for signal detection which will later be used to compare the performance of our proposed scheme. %\note[jay]{It is not clear which ones are the benchmark schemes. A better organization would be to first introduce the proposed scheme and then have a subsection for Benchmark methods}.  
\subsection{Randomized Collaboration Scheme}
A simple approach to design the collaboration matrix $\mathbf{W}$ is to use a random construction where elements of $\mathbf{W}$ are generated from a certain probability density function. In this paper, we approximate the performance of random collaboration schemes using the concept of $\delta$-Stable Embedding:
\begin{defn}{\textbf{($\delta$-Stable Embedding)}}\label{def:deltaStableEmbedding}
 \cite{Davenport_Baraniuk_2010}, A matrix $\mathbf{V} \in \mathbb{R}^{M \times N}$ satisfies the $\delta$-Stable Embedding property for $ \mathcal{U} \subset \mathbb{R}^N$ if,
 \begin{align}
 (1-\delta) \norm{\mathbf{s}_i}_2^2\leq  \norm{\mathbf{V}\mathbf{s}_i}_2^2 \leq (1+ \delta)\norm{\mathbf{s}_i}_2^2
 \end{align}
 where $\delta \in (0,1)$ and $\mathbf{s}_i \in \mathcal{U}$.
 \end{defn}Note that several random constructions guarantee that $\sqrt{\frac{M}{N}}\mathbf{P_w}$ will satisfy the $\delta$-stable embedding property with high probability. Using this concept, we state our result in the next lemma.
 \begin{lem} \label{lem:StableEmbedding}
For a random collaboration scheme $\mathbf{W}$, where $\sqrt{\frac{M}{N}} \mathbf{P_w}$ satisfies $\delta$-stable embedding property, the cumulative deflection coefficient, C-DC as given in Definition \ref{def:CDC}, can be approximated as 
   \begin{align}
    \text{C-DC}= \sum_{i=1}^I \mathbf{s}_i^T \mathbf{W}^T \left(\mathbf{W}\mathbf{W}^T \right)^{-1}\mathbf{W}\mathbf{s}_i \approx \frac{M}{N} \sum_{i=1}^I\norm{\mathbf{s}_i}_2^2 .
   \end{align}
   \end{lem}
   \begin{proof}
   The proof follows from the $\delta$-stable embedding property of Definition \ref{def:deltaStableEmbedding}. 
   \end{proof} 

%First, we consider a cost-free collaboration case and design the optimal universal collaboration strategies. Later, we consider the case where collaboration incurs a finite cost and design cost effective %universal collaboration strategies.

\subsection{Cost-Free Collaboration Strategy Design} 
In this section, we present a cost-free universal collaboration strategy, i.e., without taking into account the cost of collaboration.    
Our goal of maximizing the cumulative deflection coefficient, C-DC, can be formulated as
\begin{align}
 \text{P1:}  \quad\underset{W}{\text{maximize}} \quad  \sum_{i=1}^I \mathbf{s}_i^T\mathbf{W}^T ( \mathbf{W}\mathbf{W}^T)^{-1} \mathbf{W}\mathbf{s}_i.
 \end{align}
One direct approach to solve the optimization problem (P1) is to use semidefinite relaxation (SDR). However, such approaches are computationally expensive and cannot guarantee optimality of the solution. Furthermore, similar approaches reported in \cite{Kar_TIT_2013} and \cite{Liu_TSP_2015}, vectorize the collaboration design matrix $\mathbf{W}$ (eq. $17 (a)$ of \cite{Liu_TSP_2015}). As a consequence, we lose the ability to enforce row/column wise cost penalties. Matrix norm-based penalties are crucial for designing collaboration matrices for distributed  networks as they capture the heterogeneous aspects of the network. Interestingly, the optimization problem (P1) is equivalent to linear dimensionality reduction (from $\mathbb{R}^N$ to $\mathbb{R}^{M}$ where $M \leq N$) with a closed form solution. 
\begin{thm} \label{lem:stiefel_manifold}
The optimization problem (P1) is equivalent to Principal Component Analysis in the sense that 
\begin{small}
\begin{equation}
\underset{\mathbf{W}}{\text{max}} \;\;  \sum_{i=1}^I \mathbf{s}_i^T\mathbf{W}^T ( \mathbf{W}\mathbf{W}^T)^{-1} \mathbf{W}\mathbf{s}_i= \underset{\mathbf{W}^T \in \mathbf{S}_M^N}{\text{max}}  \text{Tr}\left( \mathbf{W} \boldsymbol{\Omega} \mathbf{W}^T \right) \nonumber
\end{equation}
\end{small} 
where, $\boldsymbol{\Omega}= \sum_{i=1}^I \mathbf{s}_i\mathbf{s}_i^T$ and $\mathbf{S}_M^N$ is the Stiefel manifold defined as $\mathbf{S}_M^N=\{\mathbf{W}^T \in \mathbb{R}^{N \times M}| \mathbf{W}\mathbf{W}^T=\mathbf{I}_M\}$.
 \end{thm}
 \begin{proof}
 To prove the lemma, first we show that we do not lose optimality if we constrain our search space so that $\mathbf{W}^T \in \mathbf{S}_M^N$. Observe that $\mathbf{P_w}=\mathbf{W}^T ( \mathbf{W}\mathbf{W}^T)^{-1} \mathbf{W}$ is a projection matrix. Using properties of projection matrices, $\left(\mathbf{P_w}\right)^2=\mathbf{P_w}$ and $\mathbf{P_w}=\mathbf{P_w}^T$ \cite{HandJ}, the objective function can be rewritten as,
\begin{align}\label{eq:lemma1}
\underset{\mathbf{W}}{\text{maximize}} \sum_{i=1}^I \norm{\mathbf{P_w}\mathbf{s}_i}_2^2.
\end{align}Now, using Gram-Schmidt orthogonalization \cite{HandJ}, we can write $\mathbf{W}^T$ as $\mathbf{W}_{ort}^T \mathbf{R}^T$, where $\mathbf{W}_{ort}\mathbf{W}_{ort}^T=\mathbf{I}_M$ and $\mathbf{R}^T$ is an upper triangular matrix. As a result,
\begin{align}
\mathbf{P_w} =& \mathbf{W}_{ort}^T \mathbf{R}^T \left( \mathbf{R} \mathbf{W}_{ort} \mathbf{W}_{ort}^T \mathbf{R}^T\right)^{-1}\mathbf{R} \mathbf{W}_{ort}\\ 
 \stackrel{(a)}{=}& \mathbf{W}_{ort}^T \mathbf{R}^T(\mathbf{R}\mathbf{R}^T)^{-1} \mathbf{R}\mathbf{W}_{ort}\\
 =& \mathbf{W}_{ort}^T \mathbf{W}_{ort}
\end{align}\
where $(a)$ follows from $\mathbf{W}_{ort}\mathbf{W}_{ort}^T=\mathbf{I}_M$. The optimization problem can then be expressed as,
\begin{align*}
\underset{\mathbf{W}}{\text{max}} \sum_{i=1}^I \norm{\mathbf{P_w}\mathbf{s}_i}_2^2 =& \underset{\mathbf{W}^T \in \mathbf{S}_M^N}{\text{max}} \sum_{i=1}^I \mathbf{s}_i^T \mathbf{W}^T \mathbf{W}\mathbf{s}_i\\
=&\underset{\mathbf{W}^T \in \mathbf{S}_M^N}{\text{max}} ~\text{Tr}\left( \mathbf{W} \boldsymbol{\Omega}\mathbf{W}^T  \right).
\end{align*}
which is equivalent to the PCA formulation.
 \end{proof}
 %The optimal collaboration strategy which maximizes the C-DC in a cost free setting is as follows.
\begin{lem}\label{lem:PCA}
The optimal solution to the C-DC maximization problem $\underset{\mathbf{W}^T \in \mathbf{S}_M^N}{\text{max}} ~\text{Tr}\left( \mathbf{W} \boldsymbol{\Omega}\mathbf{W}^T  \right)$ is given as
\begin{align}
\mathbf{W}^T_{opt}=\text{M-evecs}(\boldsymbol{\Omega}),  
\end{align} 
where M-evecs $(\boldsymbol{\Omega})$ refers to the eigenvectors corresponding to the $M$ largest eigenvalues of $\boldsymbol{\Omega}$.  
\end{lem}
We define the optimal cumulative deflection coefficient $\text{C-DC}_{\text{opt}}$ as the C-DC achieved by $\mathbf{W}^T_{opt}$ ($\text{C-DC}_{opt}$ is the C-DC obtained for cost free setting). Note that, in some specific cases the matrix $\boldsymbol{\Omega}$ can be diagonal. An example of $\boldsymbol{\Omega}$ being diagonal is when $\mathbf{s}_i$'s are of the form $\mathbf{s}_i= k_i \mathbf{e}_i$, where $k_i \in \mathbb{R}$ is an arbitrary constant and $ \mathbf{e}_i \in \mathbb{R}^N$ are the standard orthogonal basis vectors with $i^{th}$ element containing a non-zero value. In such cases, we can use the following Lemma for simplification.
\begin{lem}
If matrix $\boldsymbol{\Omega}=\sum_{i=1}^I \mathbf{s}_i\mathbf{s}_i^T$ is a diagonal matrix of rank $I$, then the optimal $\mathbf{W}=[\mathbf{W}_1 ~~\mathbf{W}_2]$, where $\mathbf{W}_1 \in M \times I$ and $\mathbf{W}_2 \in M \times (n-I)$, which maximizes the cumulative deflection coefficient C-DC, will be independent of $\mathbf{W}_2$.
\end{lem}
\begin{proof}
Let $\boldsymbol{\Omega}_I \in \mathbb{R}^{I \times I}$ denote the curtailed matrix $\boldsymbol{\Omega}$ with all zero rows and all zero columns removed. Then P1 can be written as
\begin{align*}
& \underset{\mathbf{W}}{\text{max}}  \;\;  \text{Tr}\left( [\mathbf{W}_1~~\mathbf{W}_2] ~\boldsymbol{\Omega} \left[ \begin{array}{c}\mathbf{W}_1^T \\ \mathbf{W}_2^T \end{array}   \right]    \right) =
\underset{\mathbf{W}}{\text{max}}\;\;   \text{Tr}\left( \mathbf{W}_1 \boldsymbol{\Omega}_I \mathbf{W}_1^T \right), 
\end{align*}
which is independent of $\mathbf{W}_2$.
\end{proof} 

\subsection{Cost Efficient Collaboration Strategy Design}
The proposed cost-efficient collaboration strategy design can be expressed as
 \begin{align} \label{eq:Optimization_P1}
  \quad &\underset{\mathbf{W}}{\text{maximize}} \quad  \text{Tr}\left( \mathbf{W} \boldsymbol{\Omega}  \mathbf{W}^T  \right) \\
 &\text{subject to}\quad \mathbf{W}\mathbf{W}^T = \mathbf{I}_M \nonumber\\
 &\quad \quad \quad \quad \quad \norm{\mathbf{w}_i}_{\alpha} \leq \gamma_i,~~ \text{for}~~i=\{1,2,\cdots,M\},\nonumber
 \end{align}where, $\mathbf{w}_i$ is the norm of $i$th column of $\mathbf{W}^T$ matrix and $\alpha \in \{0,1\}$ refers to the penalty imposed. Observe that, the above problem is equivalent to the sparse PCA formulation. Solving the above constrained optimization problem is difficult in its current form. Hence, we consider the following penalized collaboration  matrix design problem with $\ell_0$-pseudo norm (loosely referred to as the $\ell_0$ norm) and $\ell_1$-norm penalties, similar to the approach reported in \cite{Journee_JMLR_2010} (Section 2.3). By defining $\boldsymbol{\Omega}=\mathbf{A}^T\mathbf{A}$, the problem with $\ell_1$ and $\ell_0$ norm penalties can be rewritten as follows\footnote{For the proof of equivalence between \eqref{eq:Optimization_P1} and (P2), please see \cite{Journee_JMLR_2010}.} \footnote{For algorithmic purposes, we assume $M \leq I \leq N$.}.
 %###############################

\subsubsection{Using the $\ell_1$ norm penalty}
The modified optimization problem can be written as 
\begin{align*}
P2:  \quad &\underset{\mathbf{U,W^T}}{\text{maximize}} \quad  \text{Tr}\left( \mathbf{U}^T \mathbf{A} \mathbf{W}^T \mathbf{Y}  \right)- \sum_{i=1}^M \gamma_i \sum_{j=1}^N |w_{ij}| \nonumber\\ 
& \text{subject to} \quad \mathbf{U} \in \mathbf{S}_M^I ~~ \text{and}~~ \mathbf{W}^T \in [\mathbf{S}^N]^M .
\end{align*}
Here $\mathbf{S}_M^I$ is the Stiefel manifold, $\mathbf{Y}= \text{Diag}(y_1,\cdots,y_M)$ \footnote{Having distinct elements $y_i$ in $\mathbf{Y}$ pushes towards sparse solutions that are more orthogonal, although this is not explicitly enforced.} and 
$[\mathbf{S}^N]^M=\{ \mathbf{W}^T \in \mathbb{R}^{N \times M}| \text{Diag}(\mathbf{W}\mathbf{W}^T)= \mathbf{I}_M\}$. This problem can be decoupled in columns of $\mathbf{W}^T$ as,
\begin{align}
P2(a):  \quad &\underset{\mathbf{U}}{\text{maximize}} \quad \sum_{i=1}^m \underset{\mathbf{w}_i}{\text{maximize}} ~~y_i \mathbf{u}_i^T \mathbf{A} \mathbf{w}_i - \gamma_i \norm{\mathbf{w}_i}_1  \nonumber\\ 
& \text{subject to} \quad \mathbf{U} \in \mathbf{S}_M^I ~~ \text{and}~~ \mathbf{w}_i \in \mathbf{S}^N .
\end{align}
where, $\mathbf{u}_i$ refers to the $i$th column of vector $\mathbf{U}$ and $\mathbf{S}^N=\{\mathbf{w}_i \in \mathbb{R}^N|\mathbf{w}_i^T \mathbf{w}_i=1\}$. Notice that $\mathbf{w}_i$ refers to the column of $\mathbf{W}^T$ matrix. Using the results from \cite{Journee_JMLR_2010}, the problem can be posed in a convex form as below:
\begin{align}
P2(b): \quad &\underset{\mathbf{U}}{\text{maximize}} \quad \sum_{i=1}^{M} \sum_{j=1}^{N} \left[y_i |\mathbf{a}_j^T \mathbf{u_i}|- \gamma_i  \right]^2_+   \nonumber\\
& \text{subject to} \quad \mathbf{U} \in \mathbf{S}_M^I,
\end{align}%which is a convex function in $\mathbf{U}$ and the problem is reduced to optimizing on Stiefel manifold.

\begin{figure}[t]
\begin{center}
{\includegraphics[height=0.25\textheight, width=0.4\textwidth]{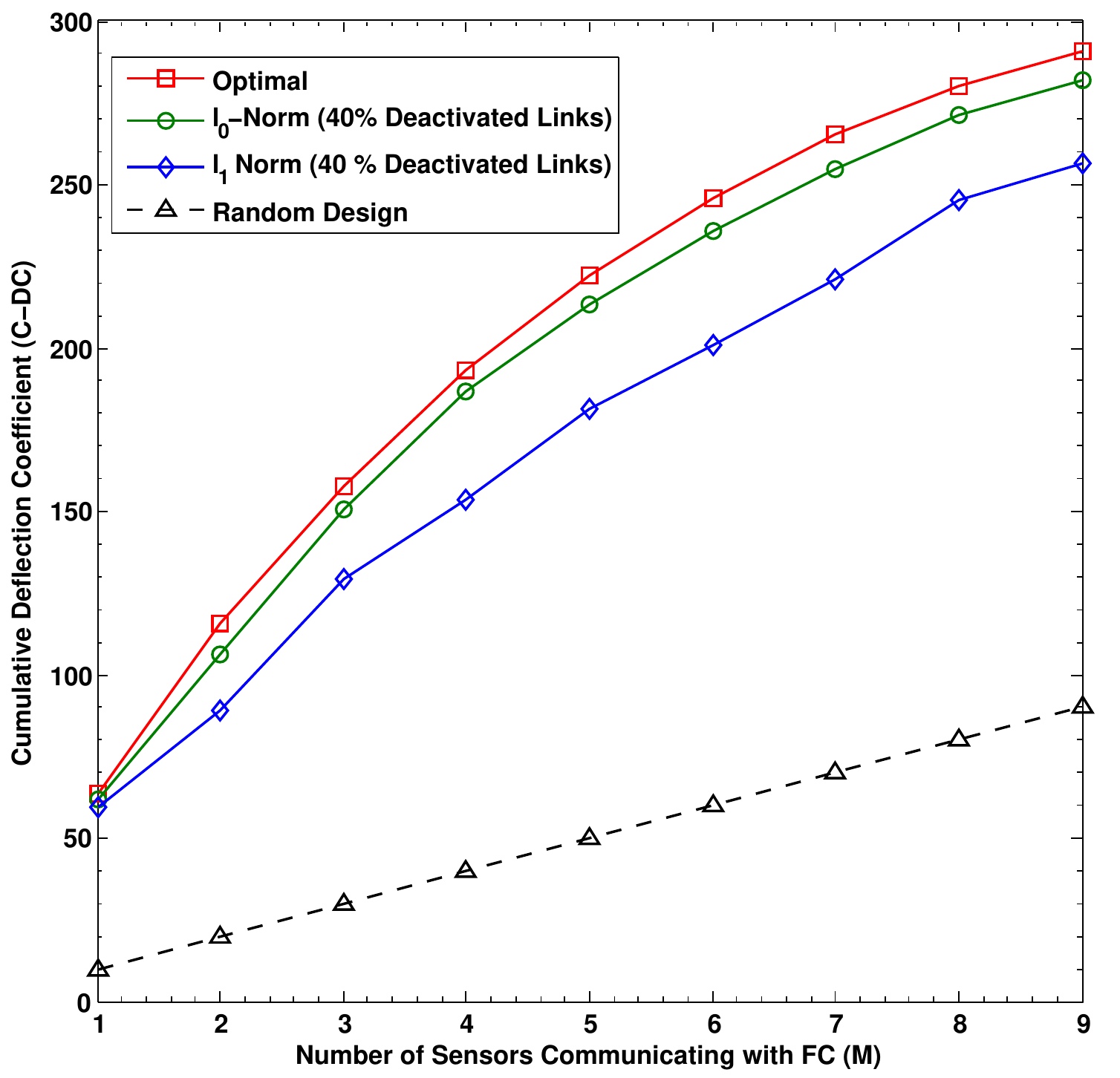}}
\caption{Cumulative deflection coefficient (C-DC) with $\ell_1$ and $\ell_0$-norm penalty vs Number of sensors capable of transmitting to FC ($M$) for, $I=10$ and $N=30$.}
\label{fig:DC_FixZerosvsM}
\end{center}
\vspace{-0.1in}
\end{figure}

\subsubsection{Using the $\ell_0$-norm penalty}
The problem can be formulated as follows,
 \begin{align*}
P3:   \quad &\underset{\mathbf{U,W^T}}{\text{maximize}} \quad  \text{Tr}\left( \text{Diag}(\mathbf{U}^T \mathbf{A} \mathbf{W}^T \mathbf{Y})^2  \right)- \sum_{i=1}^M \gamma_i \norm{\mathbf{w}_i}_0 \nonumber\\
& \text{subject to} \quad \mathbf{U} \in \mathbf{S}_M^I ~~ \text{and}~~ \mathbf{W}^T \in [\mathbf{S}^N]^M,
\end{align*}where $\norm{\mathbf{w}_i}_0$ is the norm of the $i^{th}$ column of $\mathbf{W}^T$. This problem can be decoupled in the columns of $\mathbf{W}^T$ as,
\begin{align}
P3(a):   \quad &\underset{\mathbf{U}}{\text{maximize}} \sum_{i=1}^{M} \underset{\mathbf{w}_i}{\text{maximize}} \left( y_i \mathbf{u}_i \mathbf{A} \mathbf{w}_i\right)^2 - \gamma_i \norm{\mathbf{w}_i}_0 \nonumber\\
& \text{subject to} \quad \mathbf{U} \in \mathbf{S}_M^I \text{and}~~ \mathbf{w}_i \in \mathbf{S}^N,
\end{align}
where all the notations used are as defined earlier. Again, using the results from \cite{Journee_JMLR_2010}, the problem can be posed in a convex form as below.
\begin{align}
P3(b): \quad &\underset{\mathbf{U}}{\text{maximize}} \quad \sum_{i=1}^{M} \sum_{j=1}^{N} \left[(y_i\mathbf{a}_j^T \mathbf{u_i})^2- \gamma_i  \right]_+ \nonumber\\
& \text{subject to} \quad \mathbf{U} \in \mathbf{S}_M^I.
\end{align}
While the initial formulations involved non-convex functions, we have rewritten them into a form that involve maximization of convex functions on a compact set. The dimension of the search space
is decreased enormously if the data matrix has many more columns (variables) than rows which is the case in our application of interest. We use a simple gradient-descent based approach (similar to \cite{Journee_JMLR_2010}) to solve the problems P2(b)) and (P3(b).  

\begin{figure}[t]
\begin{center}
{\includegraphics[height=0.25\textheight, width=0.4\textwidth]{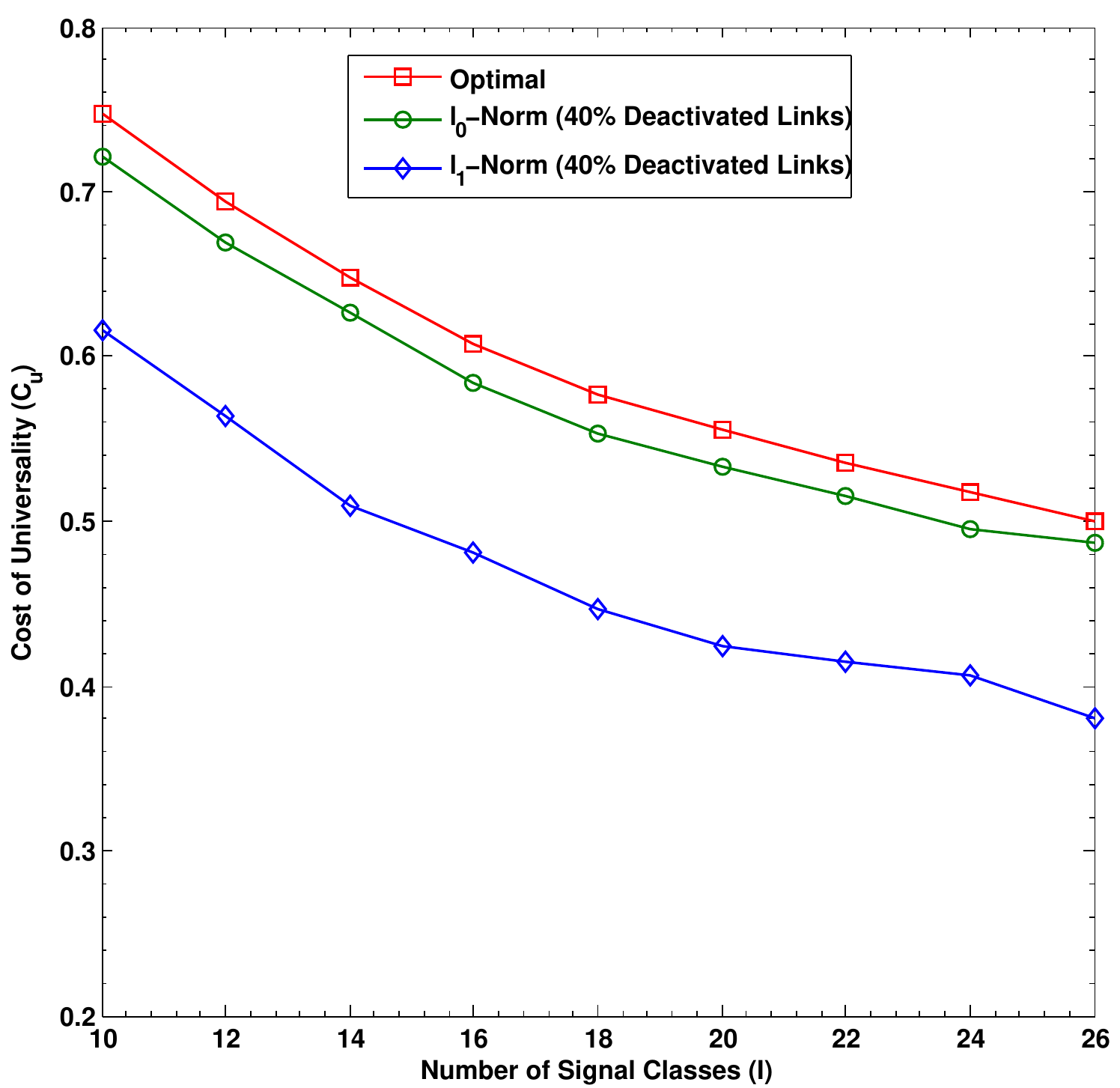}}
\caption{Cost of universality $C_u$ with $\ell_1$ and $\ell_0$-norm penalty vs Number of signal classes ($I$) for $N=30$ and $M=10$.}
\label{fig:CU_FixZerosvsP}
\end{center}
\vspace{-0.1in}
\end{figure}

\begin{figure*}[t]
  \centering
  \subfigure[$\ell_0$-norm Penalty]{
  \includegraphics[height=0.25\textheight, width=0.4\textwidth]{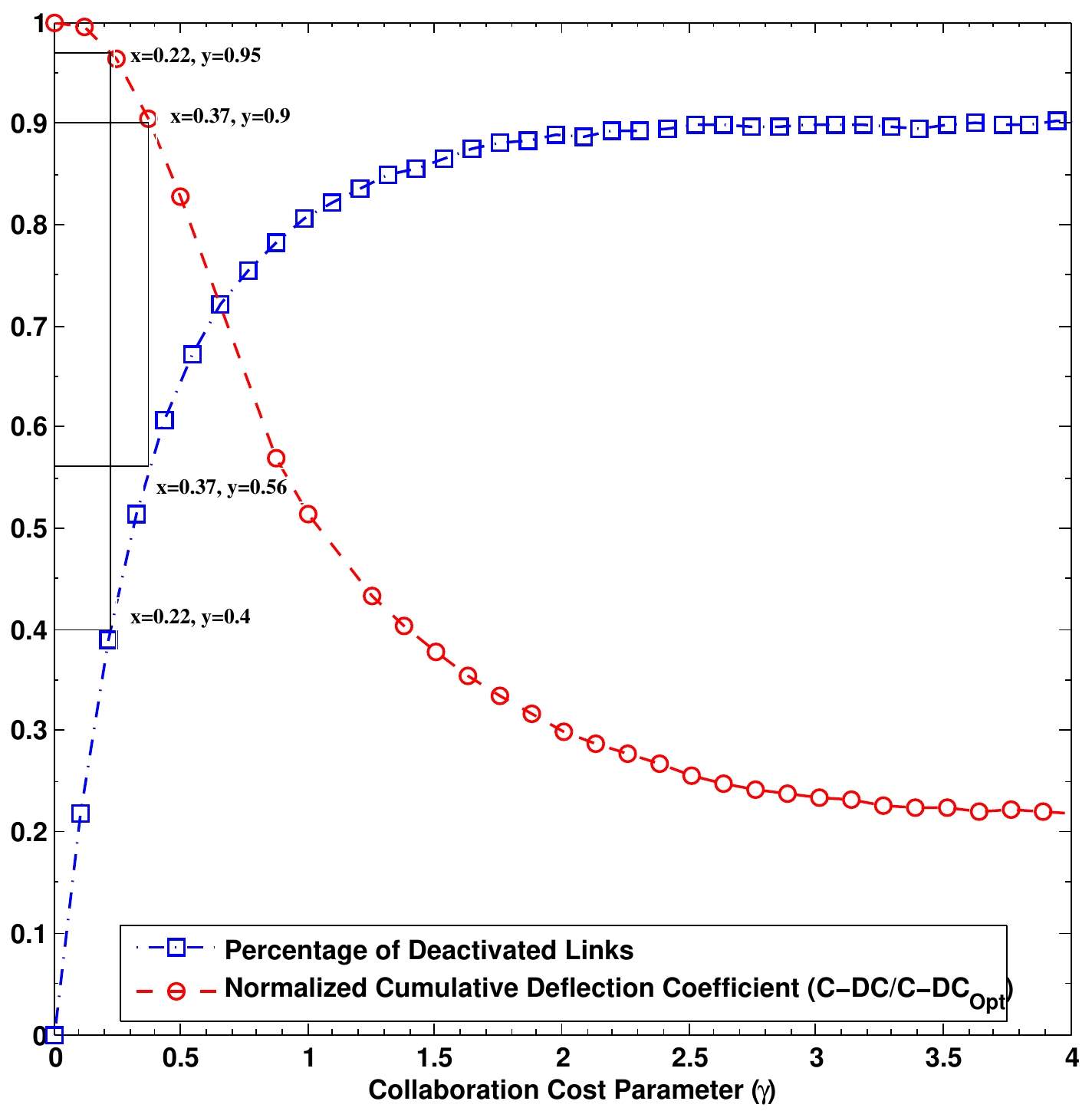} 
  \label{fig:DC_SparsityvsGammal0}}
  \quad
  \subfigure[$\ell_1$-norm Penalty]
  {\includegraphics[height=0.25\textheight, width=0.4\textwidth]{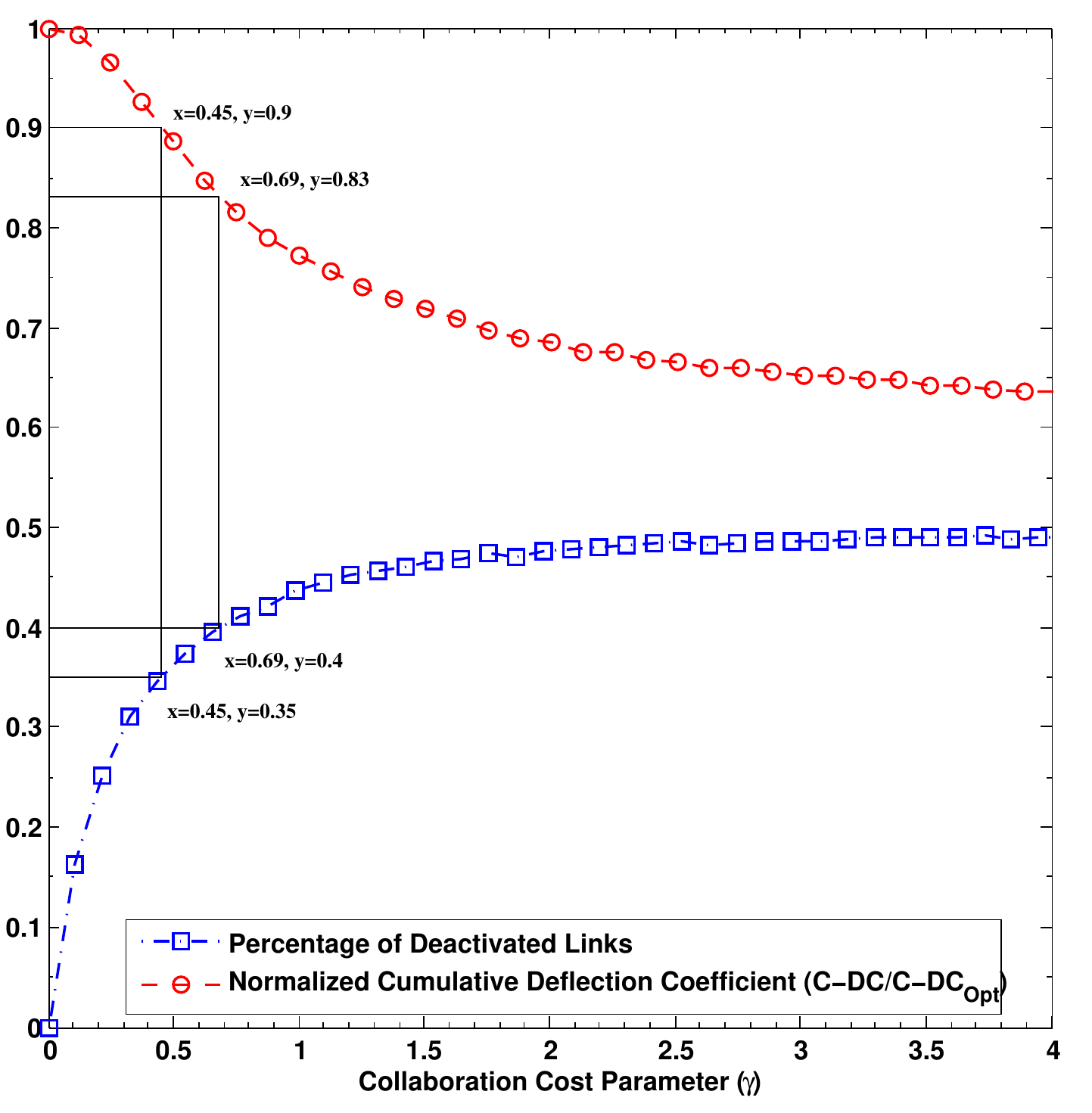}
   \label{fig:DC_SparsityvsGammal1}}
  \caption{Percentage of Deactivated links and performance with $\ell_0$ and $\ell_1$-Norm penalties, with $I=10$, $N=30$ and $M=10$.}
  \label{fig:DC_SparsityvsGamma}
  \vspace{-0.1in}
\end{figure*}
%##########################################################
\section{Results and Discussions}
In this section, we seek to answer the following questions using empirical analysis: $1)$ How much performance gain do we obtain by optimizing for the collaboration matrices? $2)$ What is the effect of dimensionality reduction ($N$ to $M$) on detection performance? $3)$ How much performance loss will we incur by considering a universal detection system for detecting a signal from the signal class $\mathcal{U}$ as opposed to optimizing a detection system for each signal independently? and, $4)$ What is the effect of the choice of the sparsity penalty function? 

We employ Monte-Carlo simulations to analyze the performance of the proposed strategies. For simplicity, we use the same cost penalty $\gamma$ for every row of the collaboration matrix $\mathbf{W}$. Observe that, for each value of $\gamma$, we obtain a specific level of sparsity, i.e,  total number of zero entries in the optimal collaboration matrix. We also assume the matrix $\mathbf{Y}$ (in $P2$ and $P3$) to be identity. Each element of the $I$ signals $\{{\mathbf{s}}_i\}_{i=1}^I$ is drawn from the standard normal distribution and each realization serves as a known signal in the set $\mathcal{U}$. 

\subsection{Impact of Collaboration on Performance}
We illustrate the performance gains obtained by introducing collaboration in Fig. \ref{fig:DC_FixZerosvsM}. In particular, we plot C-DC against the number of sensors $M$ capable of communicating with the FC, with 40$\%$ of the links deactivated $\left(\frac{\sum_{i=1}^{M} \norm{\mathbf{w}_i}_0} {M\times N}=0.4\right)$. In addition, we show the average performance achieved with randomly drawn collaboration matrix, in accordance with Lemma \ref{lem:StableEmbedding}, without any cost constraints (100$\%$ links activated). We observe that the proposed collaboration strategy performs significantly better than the random design, even with 40$\%$ of the links deactivated. 
%\note[jay]{is the optimal red curve correspond to collaboration strategy using PCA without sparsity? you should clarify that.}

\subsection{Effect of Dimensionality Reduction}
From Fig. \ref{fig:DC_FixZerosvsM}, we also notice that as $M$ decreases the C-DC also degrades. Moreover, the C-DC obtained using the $\ell_0$-norm penalty with 40$\%$ of the links deactivated is very close to the optimal C-DC ($\text{C-DC}_{opt}$) where, $\text{C-DC}_{opt}$ is the cumulative deflection coefficient achieved with zero sparsity cost penalty (100$\%$ of the links activated). We also notice that cost efficient collaboration with $\ell_0$-norm penalty performs better than $\ell_1$.

\subsection{Cost of Universality}
With the same experimental settings, we obtain the cost of universality, $C_u$, computed as in Definition \ref{def:CU} by varying the number of signals, $I$, in the class $\mathcal{U}$. As $I$ increases towards $N$, $C_u$ degrades as expected. Similar to the previous cases, using the $\ell_0$-norm produces cost of universality measures very close to the optimal case, and performs significantly better than the $\ell_1$ case.

% too. Moreover, the performance for $l_0$-Norm penalty is close to the optimal performance even with $40 \%$ of the links deactivated.

\subsection{Impact of the Sparsity Penalty Choice}
Finally, we compare the percentage of deactivated links with the normalized cumulative deflection coefficient $(\frac{\text{C-DC}}{\text{C-DC}_{\text{opt}}})$ for both $\ell_0$-norm and $\ell_1$-norm based designs. First, we consider the case where a network designer is interested in maximizing the detection performance under a certain cost budget and compare $\ell_0$-norm and $\ell_1$-norm based designs. For illustrating the comparative performance, let us consider the case where the percentage of deactivated links is fixed to be 40$\%$ for both $\ell_0$-norm and $\ell_1$-norm based designs. 
Now, from Figures \ref{fig:DC_SparsityvsGammal0} and \ref{fig:DC_SparsityvsGammal1}, we notice that the maximum detection performance in terms of normalized deflection coefficient for $\ell_0$-norm design is $0.95$ while $\ell_1$-norm design resulted in a normalized deflection coefficient of $0.83$. This pattern remains the same for different levels of sparsity. This observation suggests that the $\ell_0$-norm based design outperforms the $\ell_1$-norm based design in terms of maximizing the detection performance under a fixed cost budget.
Similarly, we consider the case where a network designer is interested in minimizing the cost of collaboration (number of communication links) while guaranteeing a certain level of detection performance. Let us consider the case where the normalized C-DC is fixed to be $0.9$ for both $\ell_0$-norm and $\ell_1$-norm based designs. We observe that for the $\ell_0$-norm based design the maximum number of links that can be deactivated is 56$\%$ in comparison to 35$\%$ in the case of $\ell_1$-based design, evidencing a similar behavior. 

%This observation suggests that the $\ell_0$-norm based design outperforms the $\ell_1$-norm based design in terms of minimizing the cost of collaboration for guaranteeing certain level of detection performance.

%\begin{figure}[t]
%\begin{center}
%{\includegraphics[width=3in,height=!]{DC_SparsityvsGammal0}}
%\caption{Fraction of number of zero components vs Collaboration cost parameter ($\gamma$).} \label{fig:DC_SparsityvsGammal0}
%\end{center}
%\end{figure}

%\begin{figure}[t]
%\begin{center}
%{\includegraphics[width=3in,height=!]{DC_SparsityvsGammal1}}
%\caption{Normalized deflection coefficient with $l_1$ and $l_0$ penalty vs Collaboration cost parameter ($\gamma$).}
%\label{fig:DC_SparsityvsGammal1}
%\end{center}
%\end{figure}
\section{Summary}
We considered the problem of designing universal collaboration strategies for high-dimensional signal detection under both cost-free and finite cost constraint models. By establishing the equivalence between collaboration matrix design and sparse PCA formulations, we adopted tools from the sparse learning literature to efficiently solve the problem. To this end, we also defined new metrics to measure performance, and quantify costs for collaboration and universality. We observed that the proposed collaboration strategies provide significant gains in detection performance in comparison to benchmark random designs. Furthermore, we demonstrated the trade-off between dimensionality reduction and the cost of collaboration ($\gamma$) to achieve desired detection performance. Finally, we analyzed the impact of the choice of sparsity penalty on the collaboration matrix design and found that the $\ell_0$-norm consistently produces superior results.

\section{Acknowledgement}
This work was supported in part by ARO under Grant W911NF-14-1-0339.
\bibliographystyle{IEEEtran}
\bibliography{IEEEabrv,References}
\end{document}